\documentclass[letterpaper, 10 pt, conference]{ieeeconf}
\IEEEoverridecommandlockouts
\usepackage{amsmath}
\usepackage{amssymb}
\usepackage{graphicx}

\usepackage{amsthm}
\usepackage{mathrsfs}

\newtheorem{thm}{\protect\theoremname}

\newtheorem{prop}[thm]{\protect\propositionname}
\newtheorem{rem}[thm]{\protect\remarkname}

\providecommand{\definitionname}{Definition}
\providecommand{\propositionname}{Proposition}
\providecommand{\remarkname}{Remark}
\providecommand{\theoremname}{Theorem}
\providecommand{\lemmaname}{Lemma}
\providecommand{\corollaryname}{Corollary}

\newcommand\scalemath[2]{\scalebox{#1}{\mbox{\ensuremath{\displaystyle #2}}}}

\begin{document}

\title{\LARGE \bf Convex Relaxations of $SE(2)$ and $SE(3)$ for Visual Pose Estimation}
\author{Matanya B. Horowitz, Nikolai Matni, Joel W. Burdick
\thanks{Matanya Horowitz, Nikolai Matni, and Joel Burdick are with the 
	Department of Control and Dynamical Systems,
        Caltech, 1200 E California Blvd., Pasadena, CA.
        The corresponding author is available at {\tt\small mhorowit@caltech.edu}}%
\thanks{Matanya Horowitz is supported by a NSF Graduate Research Fellowship.%
}
}

\maketitle
\thispagestyle{empty}
\pagestyle{empty}

\begin{abstract}
This paper proposes a new method for rigid body pose estimation based on spectrahedral representations of the tautological orbitopes of $SE(2)$ and $SE(3)$.  The approach can use dense point cloud data from stereo vision or an RGB-D sensor (such as the Microsoft Kinect), as well as visual appearance data as input. The method is a convex relaxation of the classical pose estimation problem, and is based on explicit linear matrix inequality (LMI) representations for the convex hulls of $SE(2)$ and $SE(3)$. Given these representations, the relaxed pose estimation problem can be framed as a robust least squares problem with the optimization variable constrained to these convex sets.  Although this formulation is a \emph{relaxation} of the original problem, numerical experiments indicates that it is indeed exact -- i.e. its solution is a member of $SE(2)$ or $SE(3)$ -- in many interesting settings.  We additionally show that this method is \emph{guaranteed} to be exact for a large class of pose estimation problems.

%This convex relaxation formulation can take advantage of modern semi-definite programming
%algorithms, and also allows for an $l_1$-regularization term to minimize outlier effects.
%Advantages of this approach include: 1) lack of local minima in the estimation problem (in contrast
%to ICP methods); 2) second order convergence of the algorithm; 3) minimization of outlier
%corruption, which are common in stereo vision; and 4) ease of implementation via modern semidefinite
%solvers. The approach is illustrated and compared to other methods using a pose estimation problem.
\end{abstract}

%%% -------------------------------------------------------------------------------------------- %%%
\section{Introduction} \label{sec:intro}

The pose estimation problem in computer vision is to identify a transformation that, when applied to a known object model, yields the object as perceived through the system's sensors.  Pose estimation is fundamental in robotics as it allows systems to reason about the state of the environment from their sensory data. Applications arise in robotic manipulation \cite{Hebert:2011ec,hudson2012end,Srinivasa:2010tb} object tracking \cite{Ma:2011ix}, and visual odometry \cite{pfister2002weighted}, among others. 

Many successful techniques have been proposed and validated. The typical procedure consists of identifying noteworthy features of the object model, and then identifying similar features from the sensed input. Given this correspondence, one of several methods estimates the transformation that aligns as many of these features as possible, for example \cite{lepetit2009epnp}. Typically, potential mismatches in correspondence require the optimization to be repeated, giving rise to algorithms such as iterated closest point (ICP) \cite{Besl:wl}. These techniques are typically combined with probabilistic selection of features for matching using the RANSAC algorithm.

However, many existing methods for pose estimation are primarily local-search methods that perform iterative linearization of the coordinate transform from the object's last known position via Levenberg-Marquardt \cite{Srinivasa:2010tb, Krainin:2010wu}. While such approximations are appropriate at high frame rates or slow object/robot movement, these methods require re-initialization if the object is not continually tracked successfully.  These issues can be alleviated by techniques that provide a strong prior on the object, such as those incorporating Extended Kalman Filters \cite{Ma:2011ix,Chung:2007iu,mourikis2007sc}. 

The problem of pose estimation is also often plagued by the presence of outliers in the sensory data stream, which may arise from multiple causes, such as erroneous visual depth discontinuities. Additionally, the construction of approximation algorithms further aggravates the problem of mis-correspondence. The common Best-Bin First (BBF) \cite{Beis:1997gb}, used to rapidly compute approximate nearest neighbors, creates spurious correspondences in its quest for speed. Further, there is the tradeoff in feature-detection methods, for instance in the choice between SIFT \cite{Lowe:2004kp} and SURF \cite{Bay:2008ud} features, with the former typically being more
accurate and the latter more rapid. The system designer therefore faces a number of tradeoffs between robustness and speed, contributing to the presence of outliers and uncertainty \cite{Martinez:2010jy,Hudson:auro2013}. Indeed, it is the presence of such errors that limits the use of closed form solutions to the pose estimation problem in practice \cite{Horn:87, Horn:88}.

Inspired by recent advances in convex algebraic geometry (reviewed in Section \ref{sec:review}), this paper takes a quite different approach to the pose estimation problem: pose estimation is framed as a (robust) least squares problem (Section \ref{sec:leastsquares}), where the optimization variable lies in the convex hull of either $SE(2)$ or $SE(3)$. 

By relaxing the constraint on the optimization variable (i.e. the transformation matrix), and allowing it to lie in the convex hull of the Lie group, rather than the Lie group itself, the resulting optimization is a semidefinite program (SDP), and hence convex. Formulating the pose estimation problem as a convex optimization problem is appealing as there exist a number of off-the-shelf public and commercial solvers that allow for the practitioner to focus on the design of the system rather than the details of the numerical solution method. These solvers obtain second order convergence and may therefore be quite rapid in practice. This emphasis on a convex relaxation of the $SE(n)$ constraint has also been examined in \cite{forbes}, developed independently of the current work, where the authors relax the constraint of the special euclidean group by requiring the solution to lie in the convex hull of the orthogonal group $O(n)$.

Furthermore, if the relaxation can be shown to be exact, that is to say that the optimal solution to the problem is a member of $SE(2)$ or $SE(3)$, then the procedure is guaranteed to yield the \emph{global optimum}, alleviating the need for any form of probabilistic sampling of the solution space.  Indeed, numerical experiments show that under many standard settings the relaxation is exact, and our method yields the optimal solution to the pose estimation problem.  Encouraged by this empirical success, we also show that under mild assumptions a wide class of pose estimation problems are provably exact.  In Section \ref{sec:discussion}, we comment on both promising directions for expanding the class of problems for which this method is provably exact, and on heuristics for generating a transformation matrix given a solution that is not a member of $SE(2)$ or $SE(3)$.

In the spirit of recent results in machine learning, we also show that a robust variant of the problem may be framed as an $\ell_1$-regularized optimization problem \cite{Chandrasekaran:2012hl} (Section \ref{sec:robust}). The ability to incorporate regularization techniques further distinguishes our method from existing solutions. Application of the methodology to a practical example is demonstrated in Section \ref{sec:examples}, along with a comparison to other approaches. This example shows that our approach provides  consistently better estimates across all ranges of noise, and is particularly advantageous as the sensor noise level increases.
%%% -------------------------------------------------------------------------------------------- %%%
\section{Brief Introduction to Orbitopes} \label{sec:review}

The focus of this paper is on estimating the poses of rigid objects in $SE(n)$ and $SO(n)$ for
$n=2,3$. Our analysis does not carry directly to $n>3$, and for the remainder of the paper we use $SE(n)$ to refer to only $n=2,3$. Our approach relies upon the concept of an \emph{orbitope}, which is formally the convex
hull of the orbits given by the action of a compact algebraic group $G$ acting linearly on a real
vector space. Orbits arising from such groups have the structure of a real algebraic variety and
thus, the orbitope is a \emph{convex semi-algebraic} set. These objects, whose study was formalized
and initiated in \cite{orbitopes}, are at the heart of the emerging field of convex algebraic
geometry. 

In the case of finite groups $G$, the study of such objects has a rich history. The platonic solids, the permutahedra, the Birkhoff polytope and the traveling saleseman polytopes are examples of such discretely generated orbitopes. These well studied objects have been analyzed in depth in the
context of combinatorial optimization \cite{orb29,orb38,orb25}.  Orbitopes of compact Lie groups
have proved useful in the analysis of protein structure prediction \cite{orb23}, quantum information
theory \cite{orb2} and real algebraic geometry (in \cite{orb4_b}, certain $SO(n)$ orbitopes are used
to show that there are many more non-negative polynomials than sums of squares).

We will be interested in the convex representation of orbitopes generated by a group acting on its identity element, that is to say the group's \emph{tautological orbitope}.  In particular, our method is based on recent results on the \emph{spectrahedral representation} of tautological orbitopes for the groups $SO(n)$, $n=2,3$.

A set is said to admit a spectrahedral representation if it can be described as the intersection of the cone of positive semidefinite matrices and and an affine subspace, i.e. if it can be written in terms of a linear matrix inequality (LMI).  It is possible to then optimize affine and convex quadratic functions over these sets using the well established methods of \emph{semi-definite programming} (SDP).  Additionally, if a set is second-order cone (SOC) representable -- a restricted class of spectrahedral constraints -- then SOC programming techniques may be leveraged.  These have the advantage of being significantly faster than their SDP counterparts.

%Inspired by the success and pervasiveness of convex optimization in fields such as machine learning, statistics, and control \cite{Boyd:2004uz}, we seek to extend the paradigm of ``convex thinking'' to pose estimation in computer vision -- indeed, our empirical experiments and theoretical results indicate that this approach has great potential in this problem area.

 %These sets may be represented as the intersection of the cone of positive
%semidefinite (PSD) matrices and an affine space, or equivalently, can be written in terms of a
%linear matrix inequality (LMI) \cite{orbitopes}. Given a spectrahedral representation of a set, it is
%possible to optimize over it through \emph{semidefinite programming} (SDP).  In the special case where the representation of the set can be formulated in terms of second order cone (SOC) constraints, one may use  numerically faster
%\emph{second order cone programs} (SOCP).  These convex formulations have been instrumental in solving difficult problems in fields such as control theory and communications \cite{Boyd:2004uz}.  The advent of convex programming, due to the efficiency of modern solvers, has transformed the notion of a solution to a problem. In many domains, once a problem is framed as a
%SDP, it is considered ``solved'' as it inherits the analytical and numerical benefits
%of convex analysis.  Our empirical experiments and theoretical results seem to indicate that in many interesting settings, pose estimation in computer vision may now be added to the growing list of
%tractable convex problems.

\subsection{The tautological orbitope for $SO(2)$}

We begin by noting that a simple parameterization of proper rotations in 2-dimensional space is
given by
  \begin{equation}
   SO(2)=\left\{ \left[\begin{array}{cc}
      \cos\theta & \sin\theta\\
     -\sin\theta & cos\theta
   \end{array}\right]\ :\ \theta\in[0,2\pi)\right\} .\label{eq:so-2}
  \end{equation}
Letting $x=\cos\theta$ and $y=\sin\theta$, we can equivalently rewrite this expression as
  \begin{equation}
	SO(2) = \mathcal{L}\cdot\left\{ (x,y)\in\mathbb{R}^2 \ :\ x^{2}+y^{2}=1\right\} \label{eq:so2-mapping}
\end{equation}
for a linear mapping $\mathcal{L}: \mathbb{R}^{2} \rightarrow \mathbb{R}^{2\times2}$.  In this way,
$SO(2)$ is expressed as a linear map of a simple set, i.e.  the unit sphere.  This is key as the convex hull operator, which we denote by $\text{conv}(\cdot)$, commutes with linear mappings.

Next, this constraint is {\em relaxed} by taking its convex hull. The unit sphere becomes the unit
disk, and thus the constraint is replaced with $x^{2}+y^{2}\le1$. While simple, the
computational gains realized by this step are significant. 
  \begin{align}
     \text{conv}(SO(2))
      = & \left\{ \left[\begin{array}{cc}
                     x & y\\
                    -y & x
          \end{array}\right]\ :\ x^{2}+y^{2}\leq1\right\} \nonumber \\
         = & \left\{ \left[\begin{array}{cc}
                     x & y\\
                    -y & x
             \end{array}\right]\ :\ \scalemath{.8}{
	\left[\begin{array}{cc}
            I & \left(\begin{array}{c}
              x\\
              y
          \end{array}\right)\\
    \begin{array}{cc}
      (x & y)\end{array} & 1
      \end{array}\right]}\succcurlyeq0\right\} \label{eq:conv-so2-constraint}
  \end{align}
where the last equality follows by applying the Schur complement \cite{Boyd:2004uz} to the SOC constraint $x^{2}+y^{2}\leq1$.  The latter constraint in (\ref{eq:conv-so2-constraint}) is in fact an LMI,
and the unit disk is therefore said to be {\em spectrahedrally representable}, and thus tractable to optimize over.

\begin{rem}
The orbitope of $conv\left(SO(2)\right)$ is SOC representable. These
objects have significant computational advantages over general semidefinite programs and
optimization over hundreds of variables may be done on time scales of several milliseconds
\cite{Domahidi:2013uq}.
\end{rem}

\subsection{The tautological orbitope of $SO(3)$}
An explicit representation of the tautological orbitope of $SO(3)$ is given by the following result from \cite{orbitopes}.

\begin{prop}
The tautological orbitope $\text{conv}\left(SO(3)\right)$ is a spectrahedron
whose boundary is a quartic hypersurface. In fact, a $3\times3$-matrix
$X=\left(x_{ij}\right)$ lies in $\text{conv}\left(SO(3)\right)$
if and only if it satisfies (\ref{eq:sdp_rep_so3}).
\end{prop}

%%% -------------------------------------------------------------- %%%
\begin{figure*}[t!]
\begin{centering}
\begin{equation}
\left(\begin{array}{cccc}
1+x_{11}+x_{22}+x_{33} & x_{32}-x_{23} & x_{13}-x_{31} & x_{21}-x_{12}\\*
 * & 1+x_{11}-x_{22}-x_{33} & x_{21}+x_{12} & x_{13}+x_{31}\\*
 * & * & 1-x_{11}+x_{22}-x_{33} & x_{32}+x_{23}\\*
 * & * & * & 1-x_{11}-x_{22}+x_{33}
\end{array}\right)\succeq0\label{eq:sdp_rep_so3}
\end{equation}
\par\end{centering}
\caption{Spectrahedral representation of $\text{conv}(SO(3))$ \cite{orbitopes}.  Omitted $*$ elements indicate the symmetric completion of the matrix.}
\end{figure*}
%%% -------------------------------------------------------------- %%%

\begin{proof}
The following is a slightly modified version of the original derivation
suggested by Pablo Parrilo to the authors of \cite{orbitopes}, included here for completeness. An explicit parameterization of $SO(3)$ is given by its embedding into the space of pure quaternions
(a subgroup of $SU(2)$) as
\begin{align}
SO(3) & =\left\{ U\in\mathbb{R}^{3\times3}\right.\mid\label{eq:so3-spin}\\
U(u) & =  \scalemath{0.9}{
\begin{bmatrix}2(u_{0}^{2}+u_{1}^{2})-1 & 2(u_{1}u_{2}-u_{0}u_{3}) & 2(u_{1}u_{3}+u_{0}u_{2})\\
2(u_{1}u_{2}+u_{0}u_{3}) & 2(u_{0}^{2}+u_{2}^{2})-1 & 2(u_{2}u_{3}-u_{0}u_{1})\\
2(u_{1}u_{3}-u_{0}u_{2}) & 2(u_{2}u_{3}+u_{0}u_{1}) & 2(u_{0}^{2}+u_{3}^{2})-1
\end{bmatrix}}\nonumber \\
u & \triangleq\left.\left(\begin{array}{cccc}
u_{0} & u_{1} & u_{2} & u_{3}\end{array}\right),\ \|u\|_{2}^{2}=1\right\} \nonumber 
\end{align}
Noting that each term in $U(u)$ is quadratic in elements of $u$,
we may define an auxiliary matrix 
\begin{equation}
V:=\begin{bmatrix}u_{0}^{2} & u_{0}u_{1} & u_{0}u_{2} & u_{0}u_{3}\\
u_{1}u_{0} & u_{1}^{2} & u_{1}u_{2} & u_{1}u_{3}\\
u_{2}u_{0} & u_{2}u_{1} & u_{2}^{2} & u_{2}u_{3}\\
u_{3}u_{0} & u_{3}u_{1} & u_{3}u_{2} & u_{3}^{2}
\end{bmatrix}=uu^{T}.\label{eq:V}
\end{equation}
We therefore see that if $u$ is such that $U(u)\in SO(3)$, then
$V$ is a positive semi-definite (which we denote by $V\succcurlyeq0$)
$\text{rank}$-1 symmetric matrix satifsying $\text{trace}(V)=\|u\|_{2}^{2}=1$,
and that there exists an invertible affine mapping $\mathcal{A}$
such that $U(u)=\mathcal{A}(V)$. Equivalently, we may define the
set
\[
\mathcal{V}:=\{V\succcurlyeq0\ :\ \text{rank}V=1,\ \text{trace}V=1\ \}\subset\mathbb{R}^{4\times4}
\]
and express $SO(3)$ as
\begin{equation}
SO(3)=\mathcal{\mathcal{A}}\cdot\mathcal{V}\label{eq:SO3easy}
\end{equation}
for the appropriately chosen affine map $\mathcal{A}:\mathbb{R}^{4\times4}\rightarrow\mathbb{R}^{3\times3}$.
Taking the convex hull of (\ref{eq:SO3easy}) we obtain
\begin{align}
\text{conv}(SO(3)) & =  \mathcal{A}\cdot\text{conv}\mathcal{V} \nonumber \\
 & =  \mathcal{A}\cdot\text{conv}\{V\succcurlyeq0\ |\ \text{rank}V=1,\ \text{trace}V=1\}\nonumber \\
 & =  \mathcal{A}\cdot\{V\succcurlyeq0\ |\ \text{trace}V=1\}. \label{eq:so3-conv-mapping}
\end{align}
It then suffices to invert $\mathcal{A}$ to obtain the spectrahedral representation
for $\text{conv}(SO(3))$ as given in (\ref{eq:sdp_rep_so3}).
\end{proof}

\subsection{The tautological orbitope for $SE(n)$}

A spectrahedral representation of $SE(n)$ follows naturally from one for $SO(n)$.
%Only a slight modification is necessary to extend the parameterization of $SO(n)$ to a
%parametrization of $SE(n)$.
The space of the translations is simply $\mathbb{R}^n$, which of course is already a convex set. Thus, noting that when $SE(n)$ is expressed in homogeneous coordinates, it is none other than a linear mapping of elements of $SO(n)\times\mathbb{R}^n$, we are free to take convex hulls of the individual components of the transformation matrix.  A parameterization of $S \in \text{conv}\left(SE(n)\right)$ in homogeneous
coordinates is therefore given by
\begin{equation}
S = \begin{bmatrix}
R & T \\
0 & 1
\end{bmatrix}
\label{eq:SEn}
\end{equation}
where $T\in \mathbb{R}^n$ and $R\in\text{conv}(SO(n))$.

%%% -------------------------------------------------------------------------------------------- %%%
\section{Least Squares Estimation Formulation}\label{sec:leastsquares}

We assume we are given a model of an object whose pose we wish to estimate based on sensor
data. This model consists of a vector $(m_{i})_{i=1}^{N}$ of $N$ features of a model with coordinates in
$\mathbb{R}^n$ for $n=2,3$ in the object's body-fixed reference frame.  We are then given an
observation that consists of a vector $(o_{i})_{i=1}^{N}$ of the corresponding feature coordinates
in the observation frame.  The observation frame itself may consist of a projection $P$ to a lower,
two dimensional image frame as is the case for visual data, or the data may remain in three
dimensions. In either case, we wish to solve for the object pose that is most consistent with this
observation.

The problem is general and, as framed, comprises a central component of many pose estimation
problems in robotics. All pose estimation problems consists of matching known features on a model to
observed features, whether these observations arise from an image, point cloud, or other sources.

\vskip 0.1 true in
\noindent {\bf Classical Least Squares Estimation.} 
The problem of data fitting is a common one, with the most popular approach being that least-squares
regression. In our case, the data and the model are related by a
transformation element of $SE(n)$,  with  $n\in\{2,3\}$.  For $n=2$, this leads to the optimization problem
  \begin{eqnarray}
    \min &  & \sum_{i=1}^{N}\left\Vert o_{i}-Sm_{i}\right\Vert _{2}^{2}\label{eq:se2_opt_exact}\\
    s.t. &  & S\in SE(2)\nonumber 
  \end{eqnarray}
The three-dimensionsal visual feature problem is similar but must also incorporate a projection,
$P$, onto the camera frame. The general form of the least squares optimization problem is then
 \begin{eqnarray}
   \min_{S} &  & \sum_{i=1}^{N}\left\Vert o_{i}-PSm_{i}\right\Vert _{2}^{2}
                 \label{eq:se3_opt_exact}\\
        s.t. &  & S\in SE(n)\nonumber 
  \end{eqnarray}
where $n=2,3$, and $P$ may be set to the identity, or the camera matrix, as appropriate.  Problems
(\ref{eq:se2_opt_exact}) and (\ref{eq:se3_opt_exact}) are non-convex optimization problems, and typically difficult to solve.  Beyond the issue of only finding local minima, the optimization can be slow because
$SE(n)$ is difficult to represent in a manner that is amenable to optimization, with
parametrizations typically incorporating trigonometric functions.

\vskip 0.1 true in
\noindent {\bf Convex Least Squares Estimation over $SE(n)$.} The difficulties presented by
parameterizations of $SE(n)$ in the least squares problem can be alleviated by relaxing the
problem, and replacing the Lie group constraints with their respective orbitope constraints.  In particular, by using the spectahedral representations of the orbitopes of $SE(n)$ as discussed, we may write:
  \begin{eqnarray}
     \min_{S} &  & \sum_{i=1}^{N}\left\Vert o_{i}-PSm_{i}\right\Vert _{2}^{2} \label{eq:se_n_opt_hull}\\
     s.t. &  & S\in\text{conv}\left(SE(n)\right) \nonumber
  \end{eqnarray}
for $n=2,3$, where $\text{conv}\left(SE(n)\right)$ now admits a semidefinite representation, as in
equations (\ref{eq:conv-so2-constraint}), (\ref{eq:sdp_rep_so3}) and \eqref{eq:SEn}.  

To complete the formulation, the problem data is represented in a vector form: let $O \triangleq
\left[o_{1}, \ldots, o_{N}\right]^{T}$, $M \triangleq \left[m_{1}, \ldots, m_{N}\right]^{T}$,
$\bar{P}=I_{N\times N}\otimes P$, and $\bar{S}=I_{N \times N}\otimes S$, where $\otimes$ denotes
the Kronecker Product, allowing us to rewrite optimization (\ref{eq:se_n_opt_hull}) as:
  \begin{equation}
\begin{array}{rl}
 \label{eq:vectorized_ls}
      \displaystyle\min_{\bar{S}}   & \left\Vert O-\bar{P}\bar{S}M\right\Vert _{2}^{2} \\
       s.t.   & \bar{S}\in\text{conv}\left(I\otimes SE(3)\right) 
\end{array}
  \end{equation}

\begin{rem}
It is possible to include additional feature weighting information in the optimization problem
\eqref{eq:vectorized_ls} by simply weighting $\bar{P}=C\otimes P$, and $\bar{S}=C\otimes S$ for $C$
the diagonal matrix of weights $\left\{c_1,c_2,\ldots,c_N\right\}$. This may arise, for instance, in
SIFT feature matching wherein there is a score that is calculated for each feature pair, but any
similar heuristic would be admissable.
\end{rem}

%\subsection{Convex Hull Considerations}

%In order to formulate the pose estimation problem as a SDP, we relaxed the constraint on the transformation matrix -- in particular, rather than searching over $SE(n)$, we now optimize over its convex hull.  Thus, in general, it is not true that the solution to \eqref{eq:vectorized_ls} will be an admissible transformation, as Euclidean transformations are not closed under convex combinations.  Thus an important, albeit obvious, question arises: when can we guarantee that the solution of our \emph{convex relaxation} \eqref{eq:vectorized_ls} is in fact an element of $SE(n)$.

In order to formulate the pose estimation problem as a SDP, we have relaxed the constraint on the transformation matrix -- in particular, rather than searching over $SE(n)$, we now optimize over its convex hull.  Thus, in general, it is not true that the solution to \eqref{eq:vectorized_ls} will be an admissible transformation, as Euclidean transformations are not closed under convex combinations.  Thus an important, albeit obvious, question arises: when can we guarantee that the solution of our \emph{convex relaxation} \eqref{eq:vectorized_ls} is in fact an element of $SE(n)$.

%%The central difficulty that arises out of the proposed optimization method is that the problem of
%interest, estimating elements in $SE(n)$, has been relaxed -- we have replaced the set of admissible transformations with its convex hull. The solution may therefore not in fact
%be a proper Euclidean transformation.  This arises not due to numerical issues, such as roundoff
%error, but due to the fundamental fact that our relaxed problem's feasible set now includes transformations that are not in $SE(3)$.  This observation raises the following important question: under what situations
%does this  convex relaxation produce an element of $SE(n)$?

%Recall that we have taken the Lie groups of $SO(2)$ and $SO(3)$, and relaxed the constraint set to be their convex hulls, rather than the groups themselves.  Recalling that both of these groups can be  identified with the sphere, this relaxation can be thought of as replacing the sphere with the disk.   We therefore obtain a proper transformation when the solution is on the \emph{boundary} of our constraint set, rather than the interior, and we say that our relaxation is exact.
%It is a well known consequence of convex optimization that
%if the 0-level set of the objective function is outside the constraint set, i.e.  the convex hull of
%$SE(n)$, then the solution is guaranteed to lie on the boundary \cite{Boyd:2004uz}.  

\subsection{Guarantee of Boundary Solution for $SO(n)$}

%It is well known that the translational component of the optimal solution to the least squares problem over $SE(2)$ or $SE(3)$ is given by the centroid of the observed data \cite{Horn:87} . Thus, by setting $t$ to be the appropriate centering translation, the least squares problem over $SE(n)$ for $n=2,3$ can be reduced to least squares over $SO(n)$.

\begin{thm}[Sufficient conditions for exactness] \label{thm:guarantee_boundary}
Given $n=2,3$  $O, M\in \mathbb{R}^{4N}$, and $t \in \mathbb{R}^n$, if $P\in O(n+1)$, with $O(n+1)$ the orthogonal group,  then
 the solution $R$ to the optimization problem

 \begin{eqnarray}
      \min_{R} &  & \left\Vert O-\bar{P}\bar{S}M\right\Vert _{2}^{2} \label{eq:vectorized_ls2} \\
       s.t. &  & \bar{S} = I_{N \times N}\otimes S \nonumber   \\ 
       & & S = \begin{bmatrix} R & t \\ 0 & 1 \end{bmatrix} \nonumber \\
& & R\in\text{conv}\left(SO(n)\right) \nonumber
  \end{eqnarray}
lies in $SO(n)$.
\end{thm}

\begin{proof}
Consider the non-convex optimization:
 \begin{eqnarray}
      \min_{R} &  & \left\Vert O-\bar{P}\bar{S}M\right\Vert _{2}^{2} \label{eq:non-convex} \\
       s.t. &  & \bar{S} = I_{N \times N}\otimes S \nonumber  \\
       & & S = \begin{bmatrix} R & t \\ 0 & 1 \end{bmatrix} \nonumber \\
& & R\in SO(n) \nonumber
  \end{eqnarray}
  
  Our approach will be to show that this optimization has an equivalent convex reformulation given by \eqref{eq:vectorized_ls2}. The optimization is examined in the equivalent summation form \eqref{eq:se3_opt_exact}. We first rewrite the objective so as to isolate the rotational component $R$:
    \begin{eqnarray*}
    \left\Vert o_{i}-PSm_{i}\right\Vert _{2}^{2} & = & \| o_i - P\tau -P\rho m_i \|_2^2
  \end{eqnarray*}
 where \[\rho = \begin{bmatrix} R & 0 \\ 0 & 1 \end{bmatrix}\] and $\tau = [t^T, 1]^T$. It is well known that the optimal value for $t$ is the translation that aligns the centroids of $M$ and $O$ \cite{Horn:87} and it may therefore be calculated a-priori.
  
  Recalling that for a matrix $T\in O(n+1)$, we have that $\left\Vert Tx\right\Vert _{2}=\left\Vert x\right\Vert _{2}$, and letting $v_i = o_i - P\tau$,  we may write for each $i\in\{1,\dots,N\}$,
  \begin{eqnarray*}
    \left\Vert v_{i}-P\rho m_{i}\right\Vert _{2}^{2} & = & <v_{i}-P\rho m_{i},v_{i}-P\rho m_{i}>\\
     & = & \left\Vert v_{i}\right\Vert _{2}^{2}-2<v_{i},P\rho m_{i}> \\
& & \hspace{9 mm} + <P\rho m_{i},P\rho m_{i}>\\
     & = & \left\Vert v_{i}\right\Vert _{2}^{2}-2<v_{i},P\rho m_{i}>+\left\Vert P\rho m_{i}\right\Vert _{2}^{2}\\
     & = & \left\Vert v_{i}\right\Vert _{2}^{2}-2<v_{i},P\rho m_{i}>+\left\Vert m_{i}\right\Vert _{2}^{2},
  \end{eqnarray*}
where the last equality follows from $P\rho \in O(n+1)$. As constant terms in the objective do not affect
the optimal solution, optimization \eqref{eq:non-convex} can be rewritten in terms of the \emph{linear} objective function $-2<O,\bar{P}\bar{S}M>$:
 \begin{eqnarray}
      \min_{R} &  & -2<O,\bar{P}\bar{S}M> \label{eq:non-convex2} \\
       s.t. &  & \bar{S} = I_{N \times N}\otimes S  \nonumber  \\
& & S = \begin{bmatrix} R & t \\ 0 & 1 \end{bmatrix} \nonumber \\
& & R\in SO(n) \nonumber
  \end{eqnarray}

It is a standard result of convex optimization that linear objective functions attain their minima at extreme points of the feasible set \cite{Boyd:2004uz} -- thus there is no loss in replacing the constraint $R \in SO(n)$ with $R \in \text{conv}(SO(n))$ in \eqref{eq:non-convex2}, as the set of extreme points of $\text{conv}(SO(n))$ is precisely $SO(n)$.  Rewriting the objective in its quadratic, rather than linear, form then yields \eqref{eq:vectorized_ls2}, concluding the proof.

\end{proof}

\subsection{Projection of Interior Solutions}
\label{sec:interior}
Although we have shown that the least-squares estimation over the Special Euclidean group will always produce border solutions, we will subsequently modify the optimization and lose this guarantee.  In the case where the optimal solution is not an element of $SE(n)$, a practical algorithm will require a method of quickly generating a ``nearby'' Euclidean transformation.  As a potential heuristic, we propose using a projection of the inadmissible transformation onto the admissible set.  When this projection is taken with respect to the Frobenius norm, a simple singular value decomposition (SVD) based solution exists \cite{Belta:2002euclidean}.

In particular, let a rotation $S \in \text{conv}(SO(n))$ have the SVD
\begin{equation}
S = U \Sigma V^T.
\end{equation}
Then $S' = U V^T$ is the projection of $S$ onto $SO(n)$, i.e. $S' = \arg \min \{ \|T-S\|_\text{Frobenius} : T \in SO(n)\}.$  Future work will look to quantify the performance of such heuristics.
%%% -------------------------------------------------------------------------------------------- %%%
\section{Robust Formulation}\label{sec:robust}

It is well know that the least-squares methodology is sensitive to the presence of outliers. The orbitope
optimization approach can be modified to be more resilient to outliers by framing it as a robust
least squares problem.

An advantage of using convex optimization lies in its flexibility.  The optimization problem may be
augmented with various features to improve performance beyond that obtainable by simple least
squares estimation. Typical of vision problems are the existence of spurious outliers that distort
the data, causing the estimators to diverge significantly \cite{Ma:2011ix}. It is therefore desirable that we remove the effects of such
outliers automatically.

Suppose that we wish to estimate the parameters $\beta$ such that $\left\Vert o-\beta m\right\Vert
_{2}^{2}$ is minimized for data $o,m$. We may augment the traditional $\ell_{2}$ objective with an
$\ell_{1}$ penalty term, or \emph{regularizer}, as
  \[ \min_\beta \left\Vert o-\beta m\right\Vert _{2}^{2}+\lambda\left\Vert \beta\right\Vert _{1} \]
where $\lambda>0$ is a tuning parameter, resulting in the so-called LASSO least squares problem
\cite{Tibshirani:1996um}. Interestingly, the effect of this regularizer is to drive many of the
coefficients of $\beta$ to zero, resulting in a parameter vector that is sparse.  There are many
applications where it is reasonable to expect that the true solutions are in fact sparse.  This is the
driving force behind the field of \emph{compressed sensing} \cite{Candes:2006iy}. In the pose
estimation problem it is reasonable to expect that there will exist outliers, but such outliers will
be sparse in the data, i.e.  we expect most correspondences to be accurate.

Inspired by these methods, we present an $\ell_1$-penalty based modification of optimization (\ref{eq:vectorized_ls}).   In particular, we wish to minimize $\left\Vert O-\bar{P} \bar{S} M \right\Vert _{2}^{2}$ while having a solution that is robust to outliers, i.e. $X$ contains some elements we wish to ignore. We separate the error into two components: an intrinsic error $Z_{2}$ due to sensor noise, and an error $Z_{1}$ due to incorrect correspondence (or
outliers):
   \[ O-\bar{P}\bar{S}M=Z_{1}+Z_{2} \]
The regularized problem then becomes
  \begin{eqnarray*}
     \min_{\bar{S},Z_{1},Z_{2}} &  & \left\Vert Z_{2}\right\Vert _{2}^{2} + 
                 \lambda\left\Vert Z_{1}\right\Vert _{1}\\
        s.t. &  & \bar{S}\in\text{conv}\left(I\otimes SE\left(n\right)\right)\\
             &  & O-\bar{P}\bar{S}M=Z_{1}+Z_{2},
  \end{eqnarray*}
or, after solving for $Z_{2}=(O-\bar{P}\bar{S}M)-Z_1$,
  \begin{eqnarray}
     \min_{\bar{S},Z_{1}} &  & \left\Vert (O-\bar{P}\bar{S}M)-Z_{1}\right\Vert _{2}^{2} + 
           \lambda\left\Vert Z_{1}\right\Vert _{1}\nonumber \\
           s.t. &  & \bar{S}\in\text{conv}\left(I\otimes SE\left(3\right)\right)\label{eq:regularized_opt}
   \end{eqnarray}
for some user determined $\lambda>0$. Thus in principle, $Z_{1}$ will act as a correction term in the least squares component of the objective function, removing the effect of outliers -- by adding the $\ell_{1}$ penalty and tuning parameter $\lambda$, we control the sparsity and norm of $Z_1$, and thus ensure that it does not overwhelm the correct data. Unfortunately, the lack of a linear objective causes the optimization program to lose the guarantee of being on the boundary of $\text{conv}(SO(n))$, as was the case without the regularization. %Table \ref{tab:when_exact} reviews when such guarantees are available.

\begin{rem}
Numerical experiments suggest that this method also typically gives solutions that are in $SE(n)$.  As mentioned in Section \ref{sec:interior}, we are exploring possible projection based heuristics for when our relaxation fails to be exact.
\end{rem}

% Table for guaranteed solutions.
%\begin{table}[h]
%
%\begin{center}
%\begin{tabular}{|c|c|c|}
%\hline
% & $S\in\text{conv}(SO(n))$ & $S \in SO(n)$ \\
%\hline
%Least-Squares & & \checkmark \\
%\hline
%L1 Regularization & \checkmark & \\
%\hline
%\end{tabular}
%
%\caption{Situations where the optimization program is guaranteed to produce a solution on the boundary of $\text{conv}(SO(n))$ for $n=2,3$ \label{tab:when_exact}}
%\end{center}
%\end{table}
%%% -------------------------------------------------------------------------------------------- %%%
\section{Examples} \label{sec:examples}

A benchmark pose estimation problem was used to compare the proposed method against some classical
pose estimation procedures. To perform the optimization portion of our algorithm, we used the CVX
parser \cite{cvx} in conjunction with the publicly available SDPT3 semidefinite program solver
\cite{Toh99sdpt3}.

\subsection{Three Dimensional Pose Estimation}

%% -------------------------------------------------------------- %%
\begin{figure}
\begin{centering}
\includegraphics[scale=0.45]{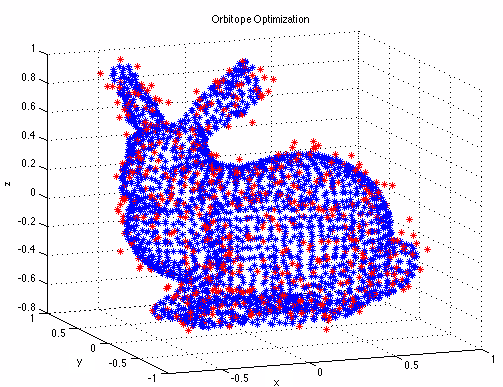}
\end{centering}
\caption{Pose estimation for a corrupted Stanford Bunny mesh. Blue points are the optimized model points while red are the observation points. Gaussian noise with zero mean and standard deviation of $\delta = .06$ was used.
was chosen for the visualization.
\label{fig:typical_bunny}}
\end{figure}
%% -------------------------------------------------------------- %%

The method was tested on a pose estimation problem involving three dimensional point-cloud data. As a
ground truth model, we used the Stanford Bunny from the Stanford Computer Graphics Laboratory,
consisting of a mesh of 944 points. The model was normalized so that the bunny had maximal extent from the origin of one unit in the $y$-coordinate from the origin.  The centroid of the model was placed at the origin, as all methods may be centered in this manner. To synthesize measurements, the model points were corrupted with varying amounts of Gaussian noise. This corruption had zero mean, standard deviation $\delta$, and was applied to the $(x,y,z)$ components of each model point individually. The
correspondence between the points was maintained in order to negate the effects of this component of
the optimization problem.  A visualization of a typical result is shown in Figure
\ref{fig:typical_bunny} for a particular noise covariance value. Performance results with varying
numbers of corrupted samples are shown in Figure \ref{ls_vary_samples}. The errors reported were the sum of the squares of the individual error distances between corresponding points, where the error was taken with respect to the true, underlying model and not the observed points. We compared our method to an
implementation using Levenberg-Marquardt, as well as an implementation that aligns the principal
components of the observed data to the model using Principle Component Analysis (PCA). These are both estimates to find a transformation to match a known model to observed points, given the correspondence between these two groups. These three methods would typically be used as part of a single iteration of ICP (in particular, usually PCA), where after the solution is computed the
correspondence problem would again be solved, and the process repeated. The methods were also tested
on varying levels of noise, with the results shown in Figure \ref{ls_vary_noise}.

%% -------------------------------------------------------------- %%
\begin{figure}
\begin{centering}
\includegraphics[scale=0.5]{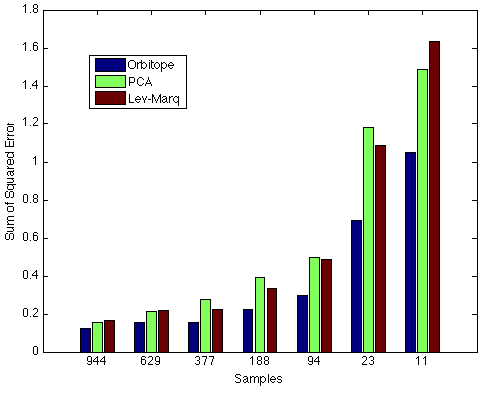}
\par\end{centering}
\caption{Mean error over twenty trials when using Orbitope, PCA, and Levenberg-Marquardt optimizations
in blue, green, and red respectively with varying numbers of randomly sampled observed points. Noise with standard
deviation $\delta=0.01$ was used. 
\label{ls_vary_samples}}
\end{figure}
%% -------------------------------------------------------------- %%

%% -------------------------------------------------------------- %%
\begin{figure}
\begin{centering}
\includegraphics[scale=0.24]{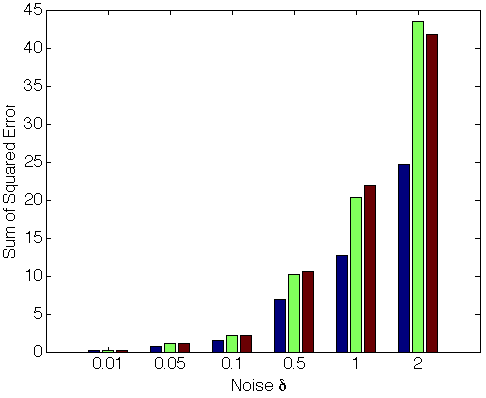}
\includegraphics[scale=0.24]{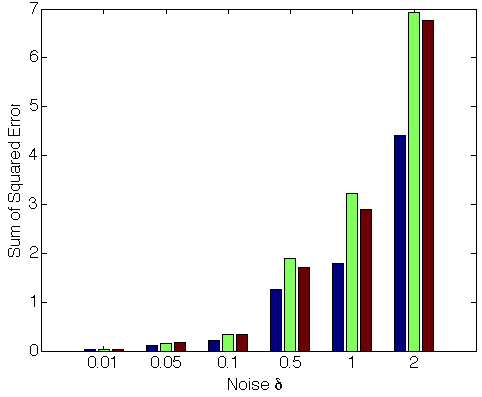}
\par\end{centering}
\caption{Mean error over twenty trials when using Orbitope, PCA, and Levenberg-Marquardt optimizations
in blue, green, and red respectively when varying noise standard deviation. Number of samples $N=23$
and $N=944$ were used in the left figures respectively. 
\label{ls_vary_noise}}
\end{figure}
%% -------------------------------------------------------------- %%

The examples were run on a 1.7Ghz Macbook Air. The Levenberg-Marquardt algorithm typically required
$25$ms for execution; aligning the principal components of the data typically required $2$ms; and
the orbitope optimization lasted $130$ms. For the range of sample points used, the execution times did not vary significantly. All methods also did not change their execution time significantly with respect to the chosen noise levels. As the correspondence was perfect, all method performed
well. However, the method proposed here provided superior solutions not only in aggregate but in the individual trials as well. For instance, in the 20 trials for which the noise was held at $\delta=0.1$, and $N=944$ points, the orbitope optimization error varied between 95\% and 23\% of the Levenburg-Marquardt error.

Next we tested the outlier rejection capabilities of the method by introducing artificial outliers
into the model.  The number of outlier points was relatively sparse compared to the total number
of model points, and they were correctly detected as outliers by the algorithm for $\lambda=0.1$ The
results are visualized in Figure \ref{fig:robust_bunny}. The error residual for the robust method
was $e=0.7825$ while that of PCA was $e=25.3098$. For these noise levels, the error of the orbitope optimization is typical (Figure \ref{ls_vary_noise}) for the case where there are no outliers, indicating the shifting of the ears has only marginal impact on the solution. Indeed, it is possible to make this gap
arbitrarily large by moving the ears further. Such situations are common in practice where, for
example, collections of points may be erroneously associated with a model, or points of a model may
not be visible due to occlusion.

% -------------------------------------------------------------- %%
\begin{figure}
\includegraphics[scale=0.5]{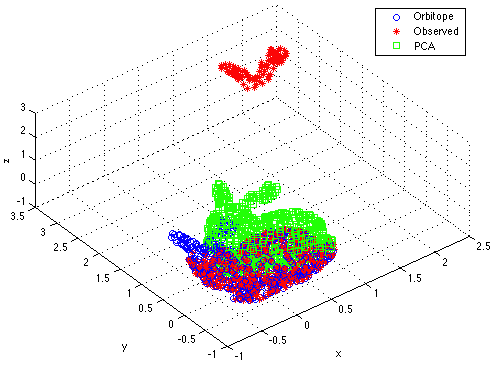}
\caption{Visualization of results for robust estimation problem. Noise of $\delta=0.01$ is added to
the normalized model, and the ears consisting of points with component $y \ge 0.6$ are translated by
$t=(2,2,2)$ to create artificial outliers. The detected data is shown in red, while the matched
model using the robust optimization is shown in blue, while the PCA match is shown in green.
\label{fig:robust_bunny}}
\end{figure}
% -------------------------------------------------------------- %%

\section{Discussion}
\label{sec:discussion}

This paper introduced a substantially new formulation of the pose estimation problem based on spectrahedral representations of the convex hulls of $SO(n)$ and $SE(n)$, for $n\in\{2,3\}$.  This reformulation allowed for a convex relaxation of the pose estimation problem, resulting in a convex least squares problem.  We also introduced an $\ell_1$ penalized variant of our problem so as to automatically remove outliers. 

We proved that under certain conditions, the solution is guaranteed
to be a proper rigid body transformation, and is therefore a \emph{globally optimal}
solution to the pose estimation problem. As mentioned, the investigation of more general conditions under which boundary solutions are guaranteed is an active area of current research, as well as the quality of projection based heuristics for when our relaxation fails to be exact.

Our numerical experiments showed that this new estimation method consistently produced more accurate
estimates, particularly at high noise levels.  Recent work \cite{Chandrasekaran:2012hl} in the
context of linear inverse problems has resulted in error bounds on regularized estimation
problems.   The application of these methods to the pose estimation problem may yield quantifiable
bounds on estimation error, and an insight as to the types of disturbances (e.g., Gaussian vs.
Poisson noise) which are best handled by the method.

The method proposed here was an order of magnitude slower than existing methods. There are a number of ways to greatly increase its speed. Our implementation relied on using CVX in conjunction with SDPT3: these are general solvers and parsers, and there is overhead accumulated. In a deployed implementation the solver would be coded directly in a language such as C++ with no parsing.  Additionally, such least squares formulations are amenable to alternating direction method of multipliers (ADMM) solutions, allowing for parallelization of the optimization procedure \cite{Boyd:admm}.  Finally, in many applications it is appropriate to use a well-informed initial guess for the solver, such as the object's last known pose. Such warm-starts are not necessary for our optimization method, but can further improve execution time.

% A number of extensions are immediate, such as: the use of the method on large images coupled with specialized solvers created for large-scale regression problems in machine learning \cite{koh2007interior}; the use of the method to provide lower noise estimates in problems where the estimates effect robot reasoning \cite{Chung:2007iu}; and the inclusion of additional constraints, such as collision between the estimated object and other objects in the scene.

\subsection{Acknowledgements}

The authors would like to thank Venkat Chandresakaran for discussion on the topics present in this
paper.  The first author is grateful for the support of a National Science Foundation graduate
fellowship.  This work was partially supported by DARPA under the ARM-S and DRC programs.

\bibliographystyle{abbrv}
\bibliography{orbitopes,se3_vision}

\end{document}